\documentclass[11pt]{article}
\usepackage[usenames]{color} 
\usepackage{amssymb} 
\usepackage{amsmath} 
\usepackage{amsthm}
\usepackage{graphicx}

\usepackage{color}
\usepackage{graphicx,subfigure,amsmath,amssymb,amsfonts,bm,epsfig,epsf,url,dsfont}
\usepackage{times}
\usepackage{bbm}      
\usepackage{booktabs}
\usepackage{cases}
\usepackage{fullpage}
\usepackage[small,bf]{caption}
\usepackage{array}
\usepackage[top=1in,bottom=1in,left=1in,right=1in]{geometry}

\newtheorem{corollary}{Corollary}

\renewcommand{\phi}{\varphi}

\renewcommand{\P}{\mathbb{P}}
\newcommand{\E}{\mathbb{E}}

\newcommand{\cL}{\mathcal{L}}

\def\ds1{\mathds{1}}
\renewcommand{\epsilon}{\varepsilon}

\newcommand{\argmax}{\mathop{\mathrm{argmax}}}

\newlength{\minipagewidth}
\newcommand{\bookbox}[1]{
\par\medskip\noindent
\framebox[\textwidth]{
\begin{minipage}{\minipagewidth}
{#1}
\end{minipage} } \par\medskip }

\newcommand{\beq}{\begin{equation}}
\newcommand{\eeq}{\end{equation}}

\newcommand{\beqa}{\begin{eqnarray}}
\newcommand{\eeqa}{\end{eqnarray}}

\newcommand{\beqan}{\begin{eqnarray*}}
\newcommand{\eeqan}{\end{eqnarray*}}

\def\ba#1\ea{\begin{align*}#1\end{align*}} 
\def\banum#1\eanum{\begin{align}#1\end{align}} 

\newcommand{\nl}{n_{\ell}}
\newcommand{\el}{\epsilon_{\ell}}

\newcommand{\il}{i_{\ell}}
\newcommand{\Al}{A_{\ell}}
\newcommand{\hatmuli}{\widehat{\mu}_i(\ell)}
\newcommand{\hatmulil}{\widehat{\mu}_{i_{\ell}}(\ell)}
\newcommand{\Hn}{\mathbf{H}}

\newtheorem{thm}{Theorem}

\def\E{\mathbb{E}}

\def\P{\mathbb{P}}
\def\1{\mathbf{1}}

\newcolumntype{M}{>{$}c<{$}}
\newcommand{\norm}[1]{\left|\left|#1\right|\right|}

\begin{document}
				\title{On Finding the Largest Mean Among Many}
		\date{}


\author{Kevin Jamieson$^\dagger$, Matthew Malloy$^\dagger$\footnote{The first two authors are listed in alphabetical order as both contributed equally.}, Robert Nowak$^\dagger$, and S\'{e}bastien Bubeck$^{\ddagger}$ \vspace{-.2cm}  \\
$^\dagger$Department of Electrical and Computer Engineering, \vspace{-.2cm} \\  University of Wisconsin-Madison \vspace{-.2cm} \\
$^\ddagger$Princeton University, \vspace{-.2cm}  \\
Department of Operations Research and Financial Engineering
  }	
		\maketitle

\begin{abstract}
Sampling from distributions to find the one with the largest mean arises in
a broad range of applications, and it can be mathematically modeled as a multi-armed bandit problem in which
each distribution is associated with an arm.
This paper studies the sample complexity of identifying the best arm (largest mean) in a multi-armed bandit problem. Motivated by large-scale applications, we are especially interested in identifying situations where the total number of samples that are necessary and sufficient to find the best arm scale linearly with the number of arms.  We present a single-parameter multi-armed bandit model that spans the range from linear to superlinear sample complexity. We also give a new algorithm for best arm identification, called PRISM, with linear sample complexity for a wide range of mean distributions. The algorithm, like most exploration procedures for multi-armed bandits, is adaptive in the sense that the next arms to sample are selected based on previous samples.  We compare the sample complexity of adaptive procedures with simpler non-adaptive procedures using new lower bounds.  For many problem instances, the increased sample complexity required by non-adaptive procedures is a polynomial factor of the number of arms.  
\end{abstract}		
		
\section{Introduction}

This paper studies the sample complexity of finding the best arm in a multi-armed bandit problem.  Consider $n+1$ arms with mean payoffs $\mu_0 > \mu _1 > \dots > \mu_n$. The mean values and the ordering of the arms are unknown.  The goal is to identify the arm with the largest mean (i.e., the ``best arm'') by sampling the arms.   A sample of arm $i$ is an independent realization of a random variable $X_i \in [0,1]$ with mean $\mu_i \in [0,1]$ (it is straightforward to extend all results presented in this paper to sub-Gaussian realizations with bounded means and variances).  

The main focus of this paper is identifying necessary and sufficient conditions under which the sample complexity (total number of samples) of finding the best arm grows linearly in the number of arms.  
This is motivated by applications involving very large numbers of arms, such as virus replication experiments testing thousands of cell strains \cite{hao:08}, cognitive radio problems searching over hundreds of communication channels \cite{haykin2005cognitive, lopez2009ofdma}, and network surveillance of large social networks. These applications are time-consuming and/or costly, so minimizing the number of samples required to find the most influential genes, best channels, or malicious agents is crucial.  This paper quantifies the minimum number of samples needed in such applications and gives a new algorithm called PRISM that succeeds using a total number of samples within a negligible factor of the minimum.

Mannor and Tsitsiklis \cite{mannor2004sample} showed that for any procedure that finds the best arm with probability at least $1-\delta$ requires on the order of $\Hn \log(1/\delta)$ samples,  where $\Hn := \sum_{i=1}^n (\mu_0-\mu_i)^{-2}$.   This lower bound shows that the sample complexity can be much greater than the number of arms.  For example, if the gap between $\mu_0$ and $\mu_1$ is $1/n$, then $\Hn$, and the sample complexity, are at least $O(n^2)$.  On the other hand, scenarios can arise in which $\Hn$ grows linearly with $n$. For instance, if $\mu_0 -\mu_i$ is greater than a positive constant for all $i$, then $\Hn=O(n)$. This case is ``sparse'' in the sense that $\mu_0$ is bounded away from all others, and recent work has shown that $O(n)$ samples are sufficient in such cases \cite{malloy2012sequential}.  

Most bandit exploration algorithms are sequential and adaptive in the sense that the selection of the arms to sample next is based on previous samples. This is necessary in order to achieve linear sample complexity in the sparse case mentioned above.  Non-adaptive methods, in which every arm is sampled an equal number of times, require at least $O(n\log n)$\footnote{In this paper our focus is on how the number of samples scales with $n$, not necessarily the probability of failure $\delta$. For instance, if we were to say an algorithm requires just $O(n)$  samples then it is understood that this many samples suffices to find the best arm with a fixed probability of error. However, in the theorem statements the dependence on $\delta$ is explicit.} samples \cite{malloy2011limits}.  
The factor of $\log n$ is significant if $n$ is large, which is one motivation for adaptive strategies like the one used in \cite{hao:08}.  Contrasting the differences between adaptive and non-adaptive sampling is a second focus of this paper.  

Of particular interest here is the scaling of the sample complexity as a function of the number of arms and the behavior of the gaps between their means.   The sparse model discussed above is an enlightening idealization, but unlikely to arise in practice.  A smoothly decaying distribution of means may be a more reasonable model for the biological and radio applications discussed above. Fig.~\ref{fig1} depicts the means in three different arm configurations.  The left plot (a) represents a sparse model in which $\mu_0$ is bounded away from all others means by a fixed constant (all gaps greater than a fixed constant).  In this case, the sample complexities of non-adaptive and adaptive strategies differs by a factor of $\log n$.  The other two plots represent cases in which the gaps between means are shrinking as $n$ increases.
In Figure~\ref{fig1}(b) $\mu_0-\mu_i = (\frac{i}{n})^{.49}$ and in (c)  $\mu_0-\mu_i = \frac{i}{n}$. The difference between the sample complexities is much more significant for these non-sparse cases.  Also note that there are non-sparse cases in which adaptive strategies can find the best arm in $O(n)$ samples like the one shown in Fig.~\ref{fig1}(b).  

\begin{figure}[h]
\centering
\centerline{
\begin{tabular}{ccc}
 \includegraphics[width=6.6cm]{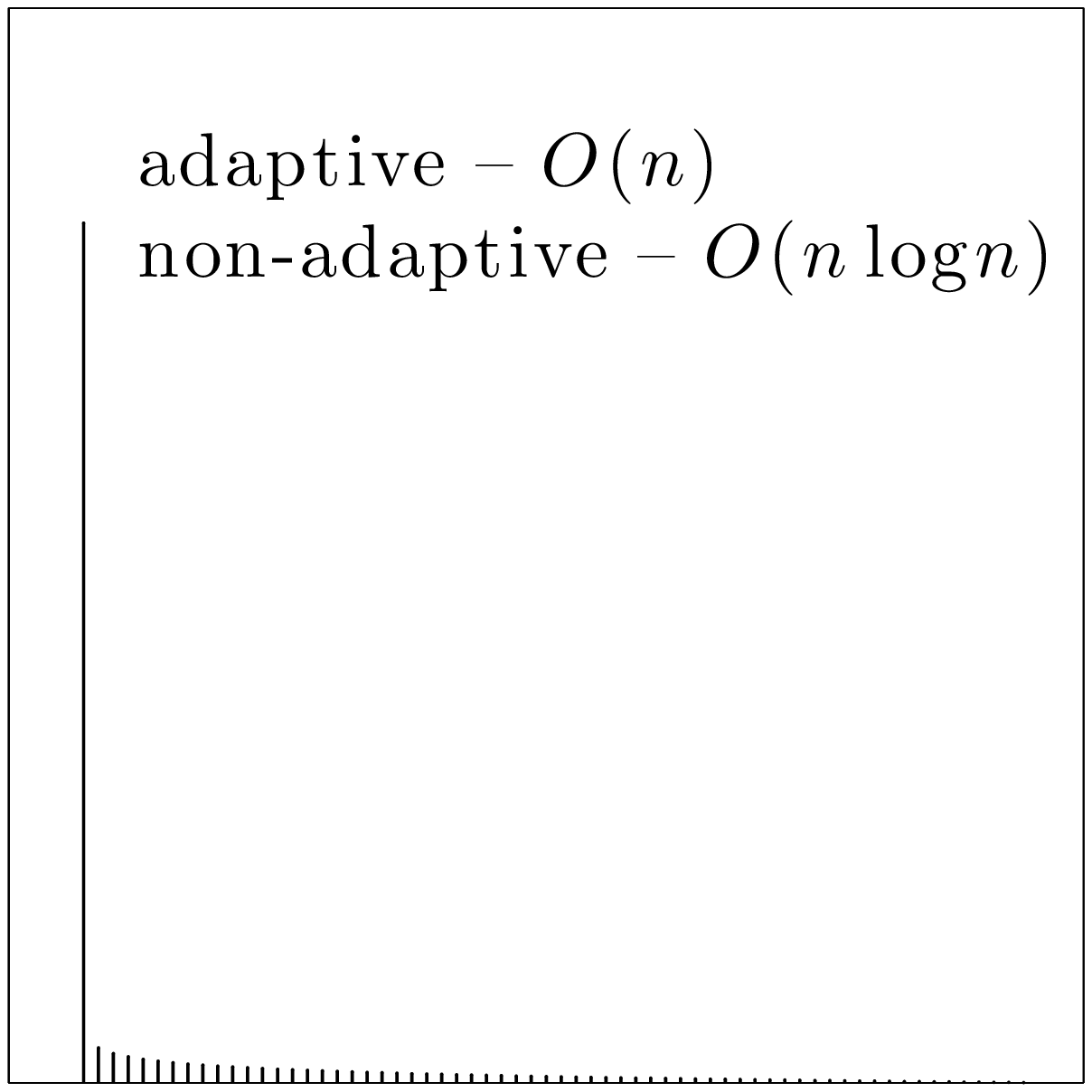} \hspace{-.7cm}  & \hspace{-.7cm}
\includegraphics[width=6.6cm]{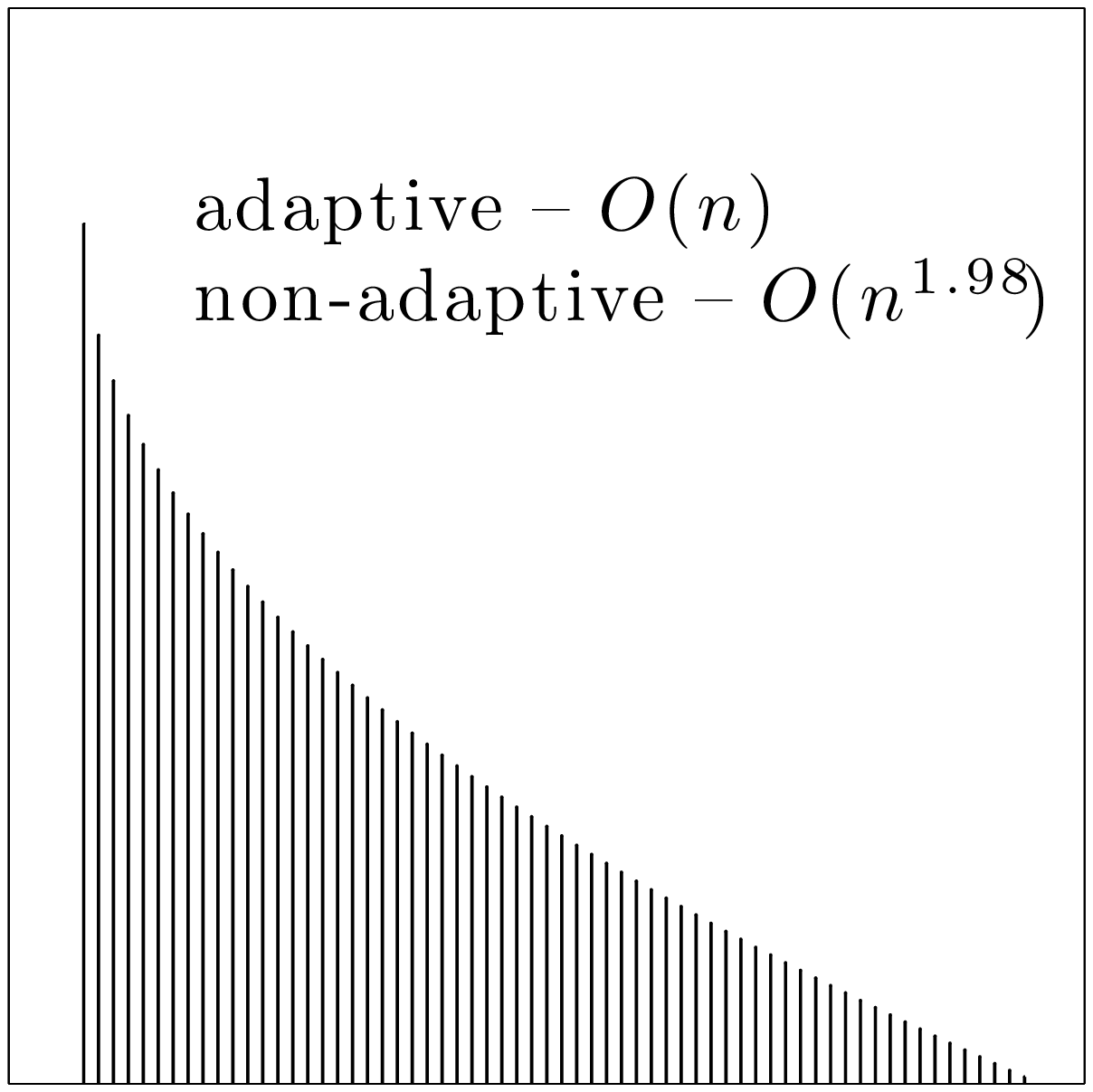} \hspace{-.7cm} & \hspace{-.7cm}
\includegraphics[width=6.6cm]{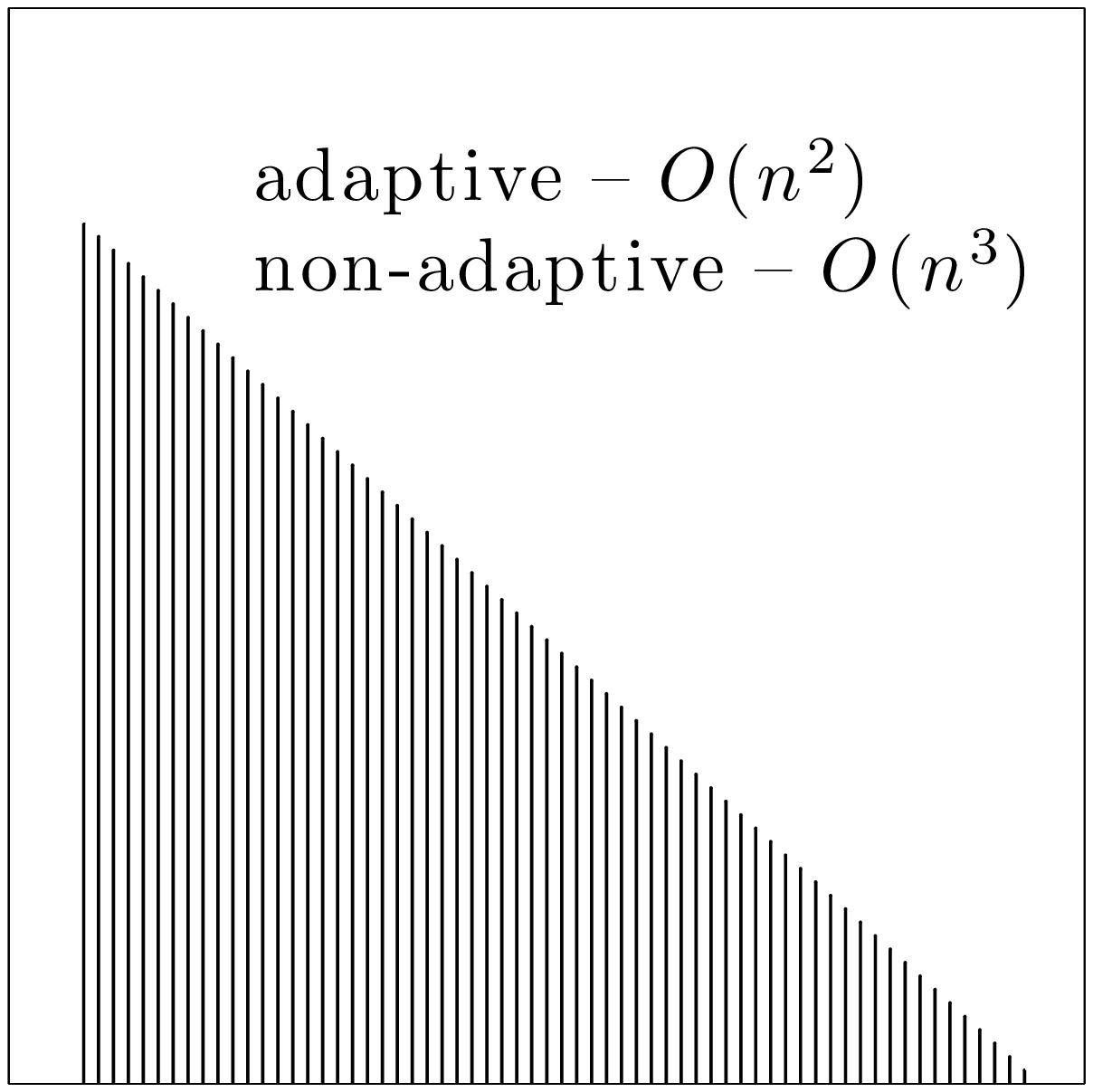}  \vspace{-.5cm} \\
(a) \hspace{-.7cm} & \hspace{-.7cm}(b)\hspace{-.7cm} &\hspace{-.7cm} (c)
\end{tabular}}
 \caption{Three different configurations of means, each ordered $\mu_0>\mu_1>\cdots >\mu_n$.
In (a) $\mu_0-\mu_i \geq 0.95$, in (b) $\mu_0-\mu_i = (\frac{i}{n})^{.49}$, and in (c)  $\mu_0-\mu_i = \frac{i}{n}$, for $i=1,\dots,n$. The necessary sample complexities (sufficient to within $\log \log n$ factors) of non-adaptive and adaptive strategies are indicated in each case.
Finding the best arm becomes increasingly difficult as the gaps between the means decrease, but in all cases adaptive strategies have significantly lower sample complexities.}
    \label{fig1}
\end{figure}

\subsection{Contributions and Organization}

The paper is organized as follows.  In Sec.~\ref{sec:alpha}, we present a single-parameter model for the distribution of means that spans the range from linear to superlinear sample complexity.  
In Sec.~\ref{sec:BAA} we present an algorithm for best arm identification with sample complexity $O(\Hn\log(1/\delta))$ for a wide range of mean distributions. In particular, we show that the algorithm has linear sample complexity for all of the single-parameter distributions satisfying $\Hn = O(n)$.  Our bounds apply to the PAC (probably approximately correct) setting (generalizing the algorithm to the fixed budget case of \cite{ABM10, BMS09}  is a challenging open problem).   While completing this paper\footnote{The main results of this paper were presented in a lecture (but not as a publication) at the Information Theory and Applications (ITA) Workshop in San Diego in February 2013.}, we became aware of independent work on the best arm problem to appear at ICML 2013 \cite{icml}.  The algorithm and theoretical analysis in that paper are essentially the same as ours, but in fact the upper bound on sample complexity bound given in \cite{icml} is slightly tighter.  

Sec.~\ref{lowerbounds} discusses the limitations of non-adaptive sampling strategies and shows that any non-adaptive sampling strategy may require drastically more samples than our adaptive algorithm.  Using new lower bounds for non-adaptive sampling procedures, we show that there exist problems in which the difference between the sample complexities of non-adaptive and adaptive procedures grows \emph{polynomially} with the number of arms.  This is somewhat surprising, as the as the advantage of adaptivity in the sparse setting (where the gaps are bound by a fixed constant) is known to be a factor of $\log n$ at best.
 To be more concrete, consider the following.  We demonstrate problem instances where adaptive procedures, for example, succeed with just $O(n)$ samples, but all non-adaptive procedures fail without at least $O(n^{1.98})$ samples.  
 This observation is crucial since it shows that adaptive designs can be vastly superior to simpler non-adaptive methods often used in practice (e.g., biological applications mentioned above). The take-away message is that the added implementational burden of adaptive sampling methods may be well worth the investment.

 Notation, in general, follows convention. Since the order and means are unknown, we denote the index of the best arm as $i^*$ throughout the paper.  An estimate of the best arm is denoted as $\widehat{i}$.  Proofs of all Theorems are found in the Appendix.

\section{A Single-Parameter Family of Mean Distributions} \label{sec:alpha}


A lower bound on the sample complexity of finding the best arm follows in a straightforward way from \cite[Theorem 5]{mannor2004sample}; see Theorem \ref{thm:AdpLB1} in Sec.~\ref{lowerbounds} for the derivation.
To find the best arm with probability at least $1-\delta$ requires at least
\begin{eqnarray} \nonumber
c_1\Hn \log(1/\delta)
\end{eqnarray}
samples, where $c_1>0$ is a universal constant. The quantity $\Hn$, refereed to as the hardness of the problem, is given by
\begin{eqnarray} \label{eqn:hardness}
\Hn = \sum_{i=1}^n \Delta_i^{-2}.
\end{eqnarray}
where
\begin{eqnarray}
\Delta_i = \mu_0 - \mu_i
\end{eqnarray}
is the gap between the best arm and the $i$th arm.   


In this paper we focus on a specific parametric family that spans the hardness of the best arm problem with a single parameter.  
Consider a model in which the means are given by
\begin{eqnarray}
\mu_i = \mu_0-(i/n)^{\alpha}
\label{alphaparam}
\end{eqnarray}
for $i=1,\dots,n$ and some $\alpha \geq 0$. We refer to this model as an \emph{$\alpha$-parameterization}.   Under this model, $\Delta_i = (i/n)^{\alpha}$.   The $\alpha$-parameterization spans the range from ``hard'' problems in which the gaps $\Delta_i$ are shrinking quickly as $n$ grows, to ``easy'' or sparse cases in which the gaps are greater than a constant (when $\alpha=0$).  Theorem \ref{thm:AdpLB1}  in Sec.~\ref{lowerbounds} yields the following lower bounds on the sample complexity of identifying the best arm (ignoring constant and $\log(1/\delta)$ factors):
\begin{eqnarray} \label{eqn:lbalph}
\mbox{sample complexity} \ \geq \  \left\{
\begin{array}{lcl}  n &  & \mbox{if $\alpha < 1/2$}  \\
n\log n  & & \mbox{if $\alpha=1/2$}  \\
n^{2\alpha} & & \mbox{if $\alpha>1/2$}.  \\
\end{array}
\right.
\label{customH2}
\end{eqnarray}

The focus of this paper is on finding sufficient conditions under which the best arm can be found with $O(n)$ samples.  If the gaps follow the $\alpha$-parameterization with $\alpha \geq 1/2$, (\ref{eqn:lbalph}) implies that the best arm \emph{cannot} be found with $O(n)$ samples.  Conversely, if the gaps satisfy $\Delta_i \geq C(\frac{i}{n})^\alpha$, for $\alpha < 1/2$, the lower bound in (\ref{eqn:lbalph}) does not preclude the possibility that order $n$ samples are sufficient to find the best arm.  In the next section, we show that when $\alpha < 1/2$, order $n$ samples are indeed sufficient.

\section{PRISM Algorithm for Best Arm Identification}
\label{sec:BAA}

To show that a linear number of pulls is sufficient for a number of problem instances, we propose and analyze the algorithm for best arm identification outlined in Fig.~\ref{fig:AME2}.  The algorithm follow a multi-phase approach, with a specific allocation of confidence and sampling budgets across phases.  The algorithm relies on the output of Median Elimination \cite{even2006action} to establish a threshold on each phase.   We mention again independent work to appear at ICML 2013 \cite{icml}, which proposes and analyzes essentially the same algorithm.

\begin{figure}[h]
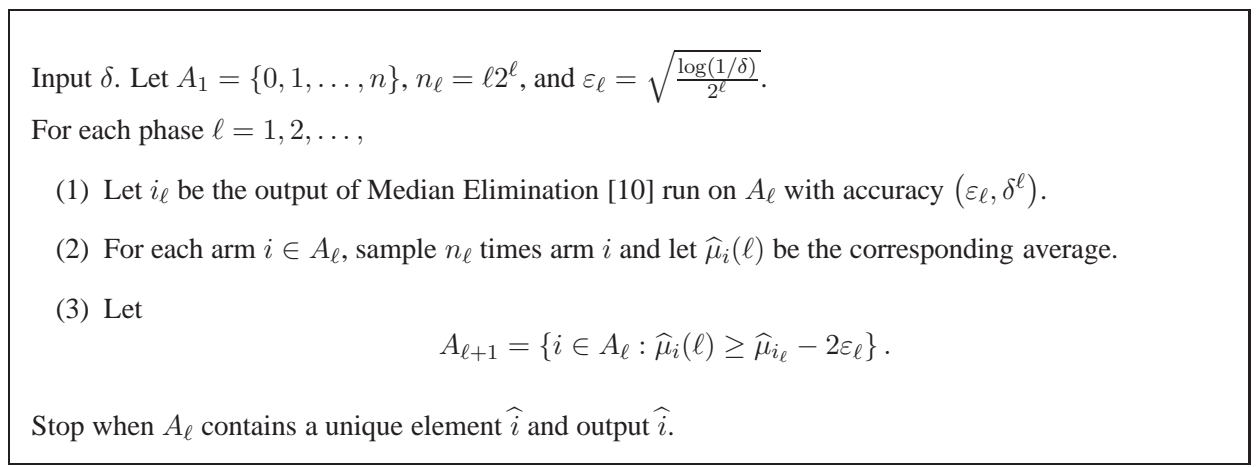

\bookbox{

\medskip\noindent
Input $\delta$.
Let $A_1=\{0,1,\hdots,n\}$, $\nl = \ell 2^{\ell }$, and $\el = \sqrt{\frac{ \log (1/\delta)}{2^{\ell }}}$.

\medskip\noindent
For each phase $\ell=1,2,\ldots,$
\begin{itemize}

\item[(1)]
Let $\il$ be the output of Median Elimination \cite{even2006action}  run on $\Al$ with accuracy $\left(\el,\delta^\ell\right)$.
\item[(2)]
For each arm $i \in \Al$, sample $\nl$ times arm $i$ and let $\hatmuli$ be the corresponding average.
\item[(3)]
Let 
$$A_{\ell+1} = \left\{i \in \Al : \hatmuli \geq \widehat{\mu}_{i_\ell}  -  2 \el \right\} .$$
\end{itemize}

\medskip\noindent
Stop when $\Al$ contains a unique element $\widehat{i}$ and output $\widehat{i}$.
}
\caption{\label{fig:AME2}
PRISM algorithm for the best arm identification problem.}
\end{figure}

The following theorem is our main result.  The sample complexity of identifying the best arm with probability at least $1-\delta$ is bounded in terms of $\Hn$ and a novel measure of complexity denoted by $\mathbf{G}:= \sum_{i=1}^n \Delta_{i}^{-2} \log_2( \Delta_{i}^{-2})$.  In general $\Hn \leq \mathbf{G} \leq \Hn \log(\Hn)$ but in many cases, ${\mathbf{G}} = \Hn$ and the bound implies that the best arm can be found using $O(\Hn)$ samples with a fixed probability of error.  This is the best known bound for the best arm problem. The proof is left to the appendix.

\begin{thm} \label{PRISM_alg}
Let $\delta \in (0,1)$.  Let $\mathbf{H} = \sum_{i=1}^n \Delta_{i}^{-2}$ where $\Delta_1$ is the minimum gap. Then with probability at least $  1 - \frac{3\delta^2}{1-\delta^2} - \frac{ \delta}{1-\delta}-\frac{4\delta^2}{(1-\delta^2)^2}$, the PRISM algorithm of Fig.~\ref{fig:AME2} stops after at most  
\begin{align*}
O\left( \log(1/\delta)  \left[ \mathbf{H}\log(\log(1/\delta)) +  \sum_{i=1}^n \Delta_{i}^{-2} \log_2( \Delta_{i}^{-2}) \right] \right)
\end{align*}
samples and outputs arm $\widehat{i} = i^*$.   
\end{thm}

\begin{corollary} \label{alphaCor}
Consider the problem instance $\mu_0=1$ and $\mu_i = 1-(i/n)^{\alpha}$, for $i=1,\dots,n$ and some $0<\alpha$. Then for a fixed probability error, using the PRISM algorithm of Fig.~\ref{fig:AME2} we have that
\begin{eqnarray} 
\mbox{number of total samples} \ = \  \left\{
\begin{array}{lcl}  O(n) &  & \mbox{if $\alpha < 1/2$}  \\
O( n\log^2 n)  & & \mbox{if $\alpha=1/2$}  \\
O(n^{2\alpha}\log(n)) & & \mbox{if $\alpha>1/2$}.  \\
\end{array}
\right.
\label{customH2}
\end{eqnarray}
 and the algorithm outputs arm $\widehat{i} = i^*$. 
\end{corollary}
\begin{proof}
For $\Delta_i = \left(\frac{i}{n}\right)^\alpha$ we have the relevant quantities given in the following table (ignoring lower order terms):
\begin{align*}
\begin{tabular}{|M|M|M|M|M|}
\hline
 & \mathbf{H} &  \sum_{i=1}^n \Delta_{i}^{-2} \log_2( \Delta_{i}^{-2})  \\
 \hline
 \alpha > 1/2 & \frac{2\alpha}{2\alpha-1} n^{2\alpha} &  \frac{(2\alpha)^2}{2\alpha-1} n^{2\alpha}\log(n)  \\
 \hline
 \alpha = 1/2 & n \log(n) & 2\alpha n \log^2 (n) \\
 \hline
 \alpha < 1/2 & \frac{2\alpha}{1-2\alpha} n  & \frac{2\alpha}{(1-2\alpha)^2} n  \\
 \hline
\end{tabular}
\end{align*}
The result follows from plugging the above quantities into Theorem~\ref{PRISM_alg}. 
\end{proof}

\noindent {\bf \emph{Conservative PRISM}.}
Consider the algorithm of Fig. \ref{fig:AME2}, but in the first line set $n_{\ell} = 2^{\ell}$ and $\epsilon_{\ell} =   \sqrt{\log(\ell^2/\delta) /2^{\ell}}$ and for item (2), run Median Elimination with input $(\epsilon_\ell, \delta/\ell^2)$.

\begin{thm} \label{th:AME}
Let $\delta \in (0,0.6]$. With probability at least $1-2 \delta - 6 \delta^2 - 6 \delta^4$, Conservative PRISM stops after at most $O\left( \Hn  \log\left( \frac{ \log ( \Hn)}{\delta} \right) \right)$ pulls and outputs $\widehat{i} = i^*$.
\end{thm}

Theorem \ref{PRISM_alg} matches the lower bound when $\alpha < 1/2$ and comes within a $\log(n)$ factor for $\alpha \geq 1/2$. On the other hand, Theorem \ref{th:AME}  comes within a factor of $\log \log n$ of the lower bound for all $\alpha > 0$. We note that the upper bound  of  \cite{icml} also comes no closer than  $\log\log n$ of the lower bound for $\alpha \geq 1/2$.

\section{Lower Bounds on the Sample Complexity of Non-Adaptive Algorithms}
\label{lowerbounds}
Here we examine the limitations of non-adaptive sampling strategies (which sample all arms an equal number of times), since these simpler procedures are not uncommon in applications like the biological problems that partially motivate this paper. A non-adaptive procedure is any procedure that samples each arm $m$ times, where $m$ is fixed a-priori, and outputs a single arm as an estimate of the best arm.   We show that all non-adaptive methods may require drastically more samples than PRISM.  The take-away message is that the small added difficulty (for the practitioner) in applying PRISM may be well worth the investment.


We begin by formally stating the adaptive lower bound developed in \cite{mannor2004sample}.

\begin{thm} {\bf{Adaptive Lower Bound}} \cite[Theorem 5]{mannor2004sample}. For every set of means, $\{\mu_{i} \}_{i=0}^n$, $\mu_{i} \in (3/8,1/2]$, there exists a joint distribution on the arms such that the arms take mean values $\{\mu_{0},\dots,\mu_{n} \}$, each arm is sub-Gaussian, and any adaptive procedure with fewer than 
\begin{eqnarray}
c_1 \mathbf{H} \log \frac{1}{8\delta}
\end{eqnarray}
samples in expectation has $ \P(\widehat{i} \neq i^*) \geq \delta$ for any $\delta \in (0, e^{-8}/8)$ and some constant $c_1$.
\label{thm:AdpLB1}
\end{thm}
\noindent Note that the restriction on the means to $(3/8, 1/2]$  in Theorem \ref{thm:AdpLB1} can be relaxed (see \cite{mannor2004sample} for details).  We proceed with the non-adaptive lower bounds which allow us to compare the best non-adaptive procedures against adaptive procedures.  

\begin{thm}  {\bf{Non-Adaptive Lower Bound}}.\label{thm:NALB1} 
Consider any $\delta \in (0,e^{-3}/24)$. For every set of means $\{\mu_1,\dots,\mu_n\}$, there exists a joint distribution on the arms such that the arms take mean values $\{\mu_1,\dots,\mu_n\}$, each arm is sub-Guassian, and any non-adaptive procedure with fewer than 
\begin{eqnarray}
 {\mathbf{H}} \log \left(\frac{n}{25  \delta} \right)  \nonumber
\end{eqnarray}
samples in expectation has $ \P(\widehat{i} \neq i^*) \geq \delta$.  Moreover, for any value of $\Hn$ there exists a joint sub-Gaussian distribution over arms with means $\{\mu_{0}, \dots, \mu_{n}\}$
satisfying $\sum_{i=1}^n (\mu_0-\mu_i)^{-2}=\Hn$, such that any non-adaptive procedure with fewer than 
\begin{eqnarray}
 \frac{\mathbf{H} n}{2} \log \left(\frac{{1}}{24\delta} \right)  \nonumber
\end{eqnarray}
samples in expectation has  $\P(\hat{i} \neq i^*) \geq \delta$. 
\end{thm}

The lower bound of Theorem~\ref{thm:NALB1} consists of two statements, the first which implies that for \emph{any} set of means, the sample complexity must be at least order $\Hn \log n$.  The second statement, on the other hand, implies the existence of particular problem instances that are especially difficult for non-adaptive procedures, requiring order $\Hn n $ samples. Inspecting the proofs of the two parts of Theorem~\ref{thm:NALB1} one sees that the minimum gap $\Delta_1$, not $\Hn$, is governing the query complexity for non-adaptive procedures. Using this fact we have the following Theorem which is proved in the appendix.

\begin{thm} {\bf{Non-adaptive Lower Bound, $\alpha$-parameterization}}. \label{cor:NA}
Consider arms with mean values according to the parameterization of (\ref{alphaparam}) for some $\alpha \geq 0$.  There exists a joint sub-Gaussian distribution on the arms such that any non-adaptive procedure with fewer than 
\begin{eqnarray} \nonumber
\begin{cases}
n \log\left(\frac{n}{25\delta}\right) \qquad &\mbox{if} \quad \alpha =0 \\
n^{2\alpha+1} \log\left(\frac{1}{24\delta} \right) \qquad &\mbox{if} \quad \alpha > 0 
\end{cases}
\end{eqnarray}
samples to find the best arm with a fixed probability of failure.
\end{thm}


To see that Theorem \ref{cor:NA} is indeed tight, it is straightforward to show that the non-adaptive procedure which chooses the arm with the largest empirical mean after sampling each arm the same number of times does indeed meet the lower bound.   Letting $m$ be the number of times each arm is sampled. Then 
\begin{eqnarray} \nonumber
\mathbb{P}\left(\widehat{i} \neq i^* \right) &\leq & \sum_{i \neq i^*} \mathbb{P}\left(\widehat{\mu}_{i^*} \leq \widehat{\mu}_{i} \right) 
\leq   \sum_{i \neq i^*} \left(\mathbb{P} \left(\widehat{\mu}_{i^*} \leq {\mu}_{i^*} - \Delta_i/2 \right) + \mathbb{P} \left(\widehat{\mu}_i  \geq \mu_{i} + \Delta_i/2 \right)   \right) \\ \nonumber
&\leq & \sum_{i \neq i^*}  2 \exp(-m \Delta_i^2)  = \sum_{i \neq i^*}  2 \exp\left(-m \left(\frac{i}{n}\right)^{2\alpha}\right)
\end{eqnarray}
which follow from a union bound and Hoeffding's inequality. For $\alpha \neq 0$, if $m \geq n^{2\alpha}$ (which implies the total number of samples is greater than $n^{2\alpha+1}$) the above sum is convergent, and the probability that the wrong arm is returned is controlled.  The case where $\alpha =0$ is also controlled if $m \geq \log n$. 


We conclude that for $\alpha \in (0,1/2)$, when compared to adaptive procedures that require just $O(n)$ samples, any non-adaptive procedure requires a factor of $n^{2\alpha}$ more samples to identify the best arm. The implications of this observation can be somewhat surprising: for many problem instances, the improvement in the sample complexity resulting from adaptivity is polynomial in $n$, compared with the typical $\log(n)$ improvement observed for sparse problems ($\alpha = 0$).



\bibliography{bestArmNips2012.bib}

\appendix
\section{Appendix}

\subsection{Proof of Theorem \ref{PRISM_alg}}
\begin{proof}
It will be useful to consider the following `slicing' of arms:
$$\Omega_s = \{i \in [n] :  5\sqrt{2}  \ \epsilon_{s+1} < \Delta_i \leq  5\sqrt{2} \ \epsilon_s\}, s \geq 1 .$$
Note that 
\begin{eqnarray} \label{eq:mal0}
25  \log \left( 1/\delta \right) \sum_{i \in \Omega_s} \Delta_i^{-2} \leq 2^{ s} |\Omega_s| \leq 50 \log \left( 1/\delta \right) \sum_{i \in \Omega_s} \Delta_i^{-2}. 
\end{eqnarray}

\noindent\textbf{Step 1: A good event.} In this step we describe the event of probability $1 - \delta$ on which we will prove the result. We want the following to hold:
\begin{align} 
& \widehat{\mu}_{i^*}(\ell) - \mu^*    \geq  -\el \ , \forall \ell \geq 1 \ ,  \label{eq:mal1} \\
& \left \vert \widehat{\mu}_{i_{\ell}}(\ell) - \mu_{i_{\ell}} \right \vert \leq \el \ , \forall \ell \geq 1 \ ,   \label{eq:mal2}  \\
& \max_{j \in \Al} \mu_j - \mu_{i_{\ell}} \leq \el \ , \forall \ell \geq 1 \ , \label{eq:mal3} \\
& \bigg| \big\{i \in \Al \cap \Omega_s : \hatmuli \geq \widehat{\mu}_{i_\ell} -  2\el \big\} \bigg| \leq \frac{|\Al \cap \Omega_{s}|}{4 } \ , \forall \ell \geq s \geq 1  . \label{eq:mal4} 
\end{align}

\noindent We first bound the probability that the above events do not hold.  By Hoeffding's inequality we have
\begin{eqnarray} \nonumber
\P \left( \hat{\mu}_{i}(\ell) - \mu_i  \geq \epsilon_\ell \right) \leq \exp(-2 n_\ell \epsilon_\ell^2) =  \delta^{2 \ell}
\end{eqnarray}
for any $i$ and note that an analogous inequality holds for the deviation away from its mean in the other direction. Thus, applying Hoeffding's and a union bound we have that the probability that $\ref{eq:mal1}) \mbox{ is not satisfied}$ is less than  $\frac{\delta^2}{1-\delta^2}$.
The probability that (\ref{eq:mal2}) is not satisfied is bound in the exact same manner with an additional factor of two to satisfy both inequality directions. The only subtlety is that after the union bound one needs to condition on the value of $\il$ before using Hoeffding's inequality, and this is possible since the random variables obtain in Step (2) of the algorithm are independent of $\il$. By the properties of median elimination and a union bound we have that the probability that (\ref{eq:mal3}) is not satisfied is less than $\frac{ \delta}{1-\delta}$. 
Observe that by \eqref{eq:mal1} and \eqref{eq:mal2} one always has:
\begin{align*}
\widehat{\mu}_{i^*}(\ell) \geq \mu^* - \el \geq \mu_{i_{\ell}} - \el \geq \hatmulil - 2 \el .
\end{align*}
which implies that the best arm is never removed from the set.

 It remains to bound the probability that (\ref{eq:mal4}) is not satisfied while (\ref{eq:mal1}),  (\ref{eq:mal2}) and (\ref{eq:mal3}) are satisfied. 
\begin{align*}
& \P \left( \bigg| \big\{i \in \Al \cap \Omega_s : \hatmuli \geq \widehat{\mu}_{i_\ell} -  2\el \big\} \bigg| > \frac{|\Al \cap \Omega_{s}|}{4} \ | (\ref{eq:mal1}), \eqref{eq:mal2}, \eqref{eq:mal3},  \Al, i_{\ell} \right) \\
& \leq \P \left( \bigg| \big\{i \in \Al \cap \Omega_s : \hatmuli \geq \mu^* -  4\el \big\} \bigg| > \frac{|\Al \cap \Omega_{s}|}{4} \ |  \Al \right) \\
& \leq \frac{4}{|\Al \cap \Omega_{s}|} \E \left( \bigg| \big\{i \in \Al \cap \Omega_s : \hatmuli \geq    \mu^* -  4\el \big\} \bigg| \ | \ \Al\right) \\
&  = \frac{4}{|\Al \cap \Omega_{s}|} \sum_{i \in \Al \cap \Omega_{s}} \P\left(\hatmuli \geq\mu^* - 4 \ \el  \ | \ \Al \right) \\
&  \leq 4 \exp( - 2 \nl \epsilon_{\ell}^2 ) = 4 \delta^{2\ell}.
\end{align*}
where the last inequality follows since all elements of $\Omega_s$ have gaps greater than $5\sqrt{2} \epsilon_{s+1} = 5 \epsilon_s$.  Summing over all $s \leq \ell$ and then over $\ell \geq 1$ gives the probability that (\ref{eq:mal4}) is not satisfied: $\sum_{\ell = 1}^{\infty} \sum_{s\leq \ell}  4\delta^{2\ell} = \frac{4\delta^2}{(1-\delta^2)^2}.$ In the next steps we will assume that \eqref{eq:mal1}-\eqref{eq:mal2}-\eqref{eq:mal3}-\eqref{eq:mal4} are satisfied as they all hold with probability at least $ 1 - \frac{3\delta^2}{1-\delta^2} - \frac{ \delta}{1-\delta}-\frac{4\delta^2}{(1-\delta^2)^2}$.

\noindent
\textbf{Step 2: Bound on the total number of phases.} It suffices to bound the number of phases given that (\ref{eq:mal1}), (\ref{eq:mal2}), (\ref{eq:mal3}) and (\ref{eq:mal4}) hold.  Let $\cL$ denote the first phase such that $| \Al | =1 $ (if there is no such phase then $\cL = +\infty$). Observe that using \eqref{eq:mal4} one can show by induction that
\begin{equation} \label{eq:mal5}
|\Al| \leq 1+\sum_{s=1}^{\ell} \frac{|\Omega_s|}{ 4^{\ell - s}} + \sum_{s=\ell+1}^{+\infty} |\Omega_s| .
\end{equation}
Define $s^* = \log_2( \Delta_1^{-2} \log(1/\delta))$ so that $\Omega_s = \emptyset$ for all $s > s^*$. By definition, when $\ell \geq s^*$ the third term in the equation immediately above is equal to zero so that
\begin{align*}
|\Al| &\leq 1+ 2^{-\ell} \sum_{s=1}^{s^*} 2^{s-s^*}2^s |\Omega_s|   \hspace{.25in} \forall \ell \geq s^*
\end{align*}
where we have
\begin{eqnarray*}
\sum_{s=1}^{s^*} 2^{s-s^*}2^s |\Omega_s| &=& \frac{\Delta_1^2 }{\log(1/\delta)} \sum_{s=1}^{\infty} 2^{2s} |\Omega_s| \\
&=&  \frac{\Delta_1^2 }{\log(1/\delta)} \sum_{s=1}^{\infty} 2^{2s} \sum_{i=1}^n \bm{1}\left\{  5\sqrt{2}  \ \sqrt{\frac{\log(1/\delta)}{2^{s+1}} } < \Delta_i \leq  5\sqrt{2} \ \sqrt{\frac{\log(1/\delta)}{2^{s}} } \right\}  \\ \nonumber
& \leq &  \frac{\Delta_1^2 }{\log(1/\delta)}  \sum_{i=1}^n \frac{50^2 \log(1/\delta)^2}{ \Delta_i^2} =  50^2 \log(1/\delta) {\bf{H}} \Delta_1^2.
\end{eqnarray*} 
We conclude that for 
\begin{align*}
\cL :=1+\max\left\{ s^*, \log_2(50^2  \log(1/\delta)  \mathbf{H} \Delta_1^2) \right\} = \log_2(2 \log(1/\delta) )+\max\left\{ \log_2(\Delta_1^{-2} ), \log_2(50^2   \mathbf{H} \Delta_1^2) \right\}
\end{align*}
 we have that  $|A_\ell| < 2$ whenever $\ell \geq \cL$. Hence, $\cL$ is an upper bound on the stopping time. 

{\noindent\textbf{Step 3: Bound on the total number of pulls.}}
Recall that  Median Elimination applied to a set $\Al$  with parameters $\el, \delta^{\ell}$ takes no more than $\frac{c_{\text{ME}}}{\el^2}|\Al| \ell \log (1/\delta) = c_{\text{ME}} \ell 2^\ell |A_\ell|$ pulls. Thus, the total number of pulls on phase $\ell$ is bounded by $c \ell 2^{\ell} |\Al|$ pulls with $c = c_{\text{ME}} + 1$. Using the results from the previous step (in particular \eqref{eq:mal5} and the stopping time $\cL$) one has that the total number of pulls is bounded from above by
\begin{eqnarray*}
\sum_{\ell=1}^{\cL} c \ell 2^{\ell}|A_\ell| &\leq & \sum_{\ell=1}^{\cL}  c\ell 2^{ \ell}\left(1+ \sum_{s=1}^{\ell}\frac{|\Omega_s|}{  4^{\ell-s}} + \sum_{s=\ell+1}^{\infty} |\Omega_s| \right) \\
& \leq & c \cL 2^{\cL+1} + c \sum_{\ell=1}^{+\infty}  \ell \sum_{s = 1}^{+\infty} \left(\frac{2^s}{2^{\ell}} \ds1_{s \leq \ell} + \frac{2^{\ell}}{2^s} \ds1_{s > \ell}\right) 2^s |\Omega_s| \\
& \leq & c \cL 2^{\cL+1} + 3 c \sum_{s = 1}^{+\infty} s 2^s |\Omega_s|   + 2c (n-1)\\
& = & c \cL 2^{\cL+1} + 150 c \log(1/\delta) \left[ \log_2(50 \log(1/\delta)) \mathbf{H} + \mathbf{G}\right] + 2c (n-1)
\end{eqnarray*}
where
\begin{align*}
\mathbf{H} \leq  \mathbf{G}:=\sum_{i=1}^n {\Delta_i^{-2}} \log_2 \left( { \Delta_i^{-2}} \right) \leq \mathbf{H} \log(\mathbf{H})
\end{align*}
which follows directly from 
\begin{eqnarray*}
 \sum_{s=1}^{\infty} s2^s |\Omega_s| &=&  \sum_{s=1}^{\infty} s 2^{s} \sum_{i=1}^n \bm{1}\left\{  5\sqrt{2}  \ \sqrt{\frac{\log(1/\delta)}{2^{s+1}} } < \Delta_i \leq  5\sqrt{2} \ \sqrt{\frac{\log(1/\delta)}{s 2^{s}} } \right\}  \\ \nonumber
& \leq  &  50 \log(1/\delta)  \sum_{i=1}^n \frac{1}{\Delta_i^2} \log_2 \left( \frac{50 \log(1/\delta) }{ \Delta_i^2} \right).
\end{eqnarray*} 
Evaluating $ \cL 2^{\cL+1}$ and collecting terms obtains the result. \end{proof}

\subsection{Proof of Theorem \ref{thm:AdpLB1}}
\begin{proof}
 Assume some procedure has $ \P(\widehat{i} \neq i^*) \leq \delta$ and requires fewer than $c_1 \mathbf{H} \log \frac{1}{8\delta}$ samples for some $\{\mu_{i} \}_{i=0}^n$, $\mu_{i} \in (3/8,1/2]$.   
This procedure is by definition $(\epsilon,\delta)$ PAC (probably approximately correct) for any $\epsilon \in (0,\Delta_1)$.   
\cite[Theorem 5]{mannor2004sample} implies any $(\epsilon,\delta)$, procedure, $\epsilon \in (0,\Delta_1)$, requires more than 
$$c_1 \sum_{i \in \mathcal{N}} \frac{1}{\mu_{i^*}-\mu_i} \log \frac{1}{8\delta}$$
samples in expectation, where 
$$\mathcal{N} = \left\{ i : \mu_i \leq \mu_{i^*} - \epsilon, \mu_i \geq \frac{\epsilon + \mu_{i^*}}{1+\sqrt{1/2}} \right\}$$
Since $\mu_{i} \in (3/8,1/2]$, $\mathcal{N} := [n]$.  Any procedure requires more than 
\begin{eqnarray} \nonumber 
c_i \Hn \log\frac{1}{8 \delta} 
\end{eqnarray}
samples in expectation.  This negates the original assumption.  
\end{proof}

\subsection{Proof of Theorem \ref{thm:NALB1}}
\begin{proof}
We restrict our attention to reward distributions of the form $\mathcal{N}(\mu_i,1)$.  Assume that $\mu_i$, $i=0,\dots,n$, are know up to a permutation, and let each arm be assigned a mean uniformly at random. We first show that the  test with minimum average probability of error simply picks the largest empirical mean among all arms, i.e., 
$\widehat{i} =\arg  \max_i \widehat{\mu}_i $,
where  $\widehat{\mu}_i = 1/m \sum_{j=1}^{m} X_{i,j}$,
and $X_{i,j}$ represents the reward of arm $i$ on the $j$th play of that arm, and $m$ is the total number of samples of each arm. This can be seen by considering the maximum a-posteriori (MAP) estimator of the best arm, which by definition has the smallest probably of error.  Under the assumption that the arms are assigned means uniformly at random, the MAP estimator reduces to the maximum likelihood (ML) estimator:
$$\widehat{i}_{\mathrm{MAP}} = \widehat{i}_{\mathrm{ML}} = \arg \max_i \; P(X_{0}^{m}, \dots, X_{n}^{m}| \mathcal{H}_i ),$$
where $X_i^{m} = X_{i,1},\dots, X_{i,m}$ and $\mathcal{H}_i$ is event that arm $i$ is the best arm.  Consider comparing between events $\mathcal{H}_i$ and $\mathcal{H}_{i'}$, $i \neq i'$: the ML test is
$P(X_{0}^{m}, \dots, X_{n}^{m}| \mathcal{H}_i ) \lessgtr_{i}^{i'} P(X_{0}^{m}, \dots, X_{n}^{m}| \mathcal{H}_{i'} )$. By the independence across arms, it is straightforward to show this test reduces to:
\begin{eqnarray}  \label{eqn:testML3}
P(X_{i}^m| \mathcal{H}_i )  P(X_{i'}^m| \mathcal{H}_i )  \lessgtr_{i}^{i'}  P(X_{i}^m| \mathcal{H}_{i'} )  P(X_{i'}^m| \mathcal{H}_{i'} ) 
\end{eqnarray}
The distribution of $X_i^m$, given $\mathcal{H}_i$, is simply
\begin{eqnarray}  \label{eqn:distA}
P(X_{i}^m| \mathcal{H}_i )= \frac{1}{\sqrt{2 \pi}} \exp\left(-\norm{X_{i}^m - \mu_{i^*}\1}^2/2\right)
\end{eqnarray}
where $X_i^m =[X_{i,1}, \dots , X_{i,m} ]^T  \in \mathbb{R}^m$.  The marginal distribution on arm $i'$, given that arm $i$ is the largest, follows a mixture distribution:
\begin{eqnarray}  \label{eqn:distB}
 P(X_{i'}^m| \mathcal{H}_i )  = \frac{1}{n\sqrt{2 \pi}} \sum_{j = 1 }^n  \exp\left(-\norm{X_{i'}^m - \mu_{j} \1}^2/2\right).
\end{eqnarray}
Combining (\ref{eqn:testML3}), (\ref{eqn:distA}), and (\ref{eqn:distB}), after a number of straightforward manipulations, excluded for brevity, it can be shown that the ML estimate prefers arm $i$ to $i'$ if and only if $\sum_{\ell=1}^{m} X_{i,\ell}  > \sum_{\ell=1}^{m} X_{i',\ell} $, or equivalently, $\widehat{\mu}_i > \widehat{\mu}_{i'}$. The estimate with minimum probability of error is simply $\widehat{i} = \arg \max_i \widehat{ \mu}_i.$

  We continue by bounding the probability of error of the maximum likelihood test.  For any estimator,
\begin{eqnarray} \nonumber
\P(\widehat{i} \neq i^*) &\geq &\P\left( \bigcup_{i\neq i^*} \widehat{\mu}_{i^*} -  \widehat{\mu}_i \leq 0 \right) \\ \nonumber
&=& \P(\widehat{\mu}_{i^*} \geq \mu_{i^*}) \; \P\left( \left. \bigcup_{i\neq i^* } \widehat{\mu}_{i^*} -  \widehat{\mu}_i \leq 0 \right \vert  \widehat{\mu}_{i^*} \geq \mu_{i^*} \right)  \\ \nonumber
&& \qquad \quad + \quad \P(\widehat{\mu}_{i^*} < \mu_{i^*}) \; \P\left( \left. \bigcup_{i\neq i^*} \widehat{\mu}_{i^*} -  \widehat{\mu}_i \leq 0 \right \vert  \widehat{\mu}_{i^*} < \mu_{i^*} \right)  \\ \nonumber 
&\geq &   \frac{1}{2} \; \P\left( \bigcup_{i\neq i^*}   \widehat{\mu}_i  \geq {\mu}_{i^*}   \right)
 =   \frac{1}{2}\left(1 -   \P\left(  \bigcap_{i\neq i^*}  \widehat{\mu}_i  \leq {\mu}_{i^*}   \right) \right)  \\
& = & \frac{1}{2}\left(1 -   \prod_{i\neq i^*}   F_{\mathcal{N}}\left(\sqrt{m \Delta_{i}^2}\right)    \right)  \label{eqn:pemidstep1} 
 \geq \frac{1}{2}\left(1 -  \prod_{i \neq i^*}  \left(1 -\frac{1}{12} \exp\left({ - m \Delta_{i}^2 }\right)\right) \right) \\ \label{eqn:LBPemin1}
&\geq& \min_{\Delta_{1},\dots,\Delta_{n}:\sum_{i}\frac{1}{\Delta_i^2} = \mathbf{H} }  \  \frac{1}{2}\left(1 -  \prod_{i \neq i^*}  \left(1 -\frac{1}{12} \exp\left({ - m \Delta_{i}^2 }\right)\right) \right)
\end{eqnarray}
where $F_{\mathcal{N}}(x)$ is the standard Gaussian CDF. The inequality in (\ref{eqn:pemidstep1}) follows since 
$F_{\mathcal{N}}(x) \leq 1- \exp(-x^2) / 12$ for $x \geq 0$  \cite[Eqn. 13]{6305026}.
The next step in the proof will be showing that error probabilities smaller than a fixed constant, (\ref{eqn:LBPemin1}) is minimized when the gaps are equal, i.e., when $\Delta_{1} = \Delta_{2} = ... = \sqrt{{n}/{\mathbf{H}}}$.  
First define $ \mathbf{z} \in \mathbb{R}^n_+$ with elements $z_i := 1/(m\Delta_{i}^2)$.
We can recover the minimum of (\ref{eqn:LBPemin1}) by solving  
\begin{eqnarray} 
\argmax_{\mathbf{z} \in \mathbb{R}^n_+ : \mathbf{1}^T \mathbf{z} = \mathbf{H}/m}  \sum_{i=1}^n \log \left(1-\frac{1}{12}\exp \left( -1 / z_{i}\right) \right). \label{eqn:equal1n}
\end{eqnarray}
Define the Lagrangian of (\ref{eqn:equal1n}) as 
\begin{eqnarray} \nonumber
\mathcal{L}(\mathbf{z}, \lambda) = -\sum_{i=1}^n \log \left(1-\frac{1}{12}\exp \left(  z_{i}^{-1}\right) \right) - \lambda(\mathbf{1}^T \mathbf{z}- \mathbf{H}/m).
\end{eqnarray}
From \cite[p. 321]{nocedal2006numerical}, any $\mathbf{z}$ that maximizes (\ref{eqn:equal1n}) necessarily satisfies 
\begin{eqnarray} \label{eqn:lagsys1}
\frac{\partial{\mathcal{L}}}{ \partial {z_i}} &=& \frac{ z_i^{-2}}{12 \exp (z_i^{-1}) - 1} - \lambda z_i = 0 \quad \forall \; i  \\ \nonumber
\mathbf{1}^T \mathbf{z} &=& \mathbf{H}.
\end{eqnarray}
The above system of equations is satisfied by pairs $(\mathbf{z}, \lambda)$ that satisfy
\begin{eqnarray} \label{eqn:lagsol1}
\left\{z_i:\lambda = \frac{z_i^{-3}}{12 \exp (z_i^{-1}) - 1} \right\} \quad \forall \; i
\end{eqnarray}
and $\mathbf{1}^T \mathbf{z} = \mathbf{H}$ simultaneously.  Differentiation of (\ref{eqn:lagsol1}) shows the function $\lambda(z_i)$ is monotonically increasing in $z_i$ for $z_i \leq 1/3$.  First, consider a solution to (\ref{eqn:lagsys1}) which has one or more $z_i \geq 1/3$.  This would imply $m\Delta_{i}^2 \leq 3$ for some $i$, and from (\ref{eqn:LBPemin1}), $\P(\widehat{i} \neq i^*) \geq \tfrac{1}{24} \exp(-3)$.
For any $\lambda$, (\ref{eqn:lagsol1}) is satisfied by at most one $z_{i} \in (0,1/3]$ by the monotonicity of the function on this range; this implies implies either \emph{1)} the $\mathbf{z}$ that maximizes (\ref{eqn:equal1n}) has the form $z_1 = z_2 = \dots = z_n$, or \emph{2)} $\P(\widehat{i} \neq i^*) \geq \tfrac{1}{24} \exp(-3)$.
We focus our attention on the case when $z_1 = \dots =z_n$ (and thus $\Delta_{1} = \dots = \Delta_{n}$). Since $\sum_i {1}/{\Delta_{i}^2} = \mathbf{H}$, $\Delta_{i} = \sqrt{n/\mathbf{H}}$ for all $i$.  (\ref{eqn:LBPemin1}) gives
\begin{eqnarray} \nonumber
\P(\widehat{i} \neq i^*) &\geq& \frac{1}{2}\left(1 -  \left(1 -\frac{1}{12} \exp\left({ - \frac{mn}{\mathbf{H}}  }\right)\right)^n \right).
\end{eqnarray}
Recall the total number of samples is given by $nm$.  If $mn \leq  \mathbf{H} (\log n + \log \left( (25\delta)^{-1})\right)$, then for $\delta \in (0, e^{-3}/24 )$
\begin{eqnarray}
\P(\widehat{i} \neq i^*) &\geq& \frac{1}{2}\left(1 -  \left(1 -\frac{25  \delta}{12n}  \right)^n \right) \\
&\geq &\frac{1}{2}\left(1 - \exp\left(-\frac{25  \delta}{12}\right)\right) \quad \mbox{for all $n\geq 1$} \\
&\geq & \delta
\end{eqnarray}
which completes the proof of the first statement of the theorem.

To prove the second statement of the theorem, consider the following set of gaps  -- 
 \begin{eqnarray} \nonumber
 \Delta_{i}  = \begin{cases} 
 \sqrt{\frac{2}{\mathbf{H}}} & i = 1 \\
\sqrt{\frac{2(n-1)}{\mathbf{H}}} & i > 1 .
\end{cases}
 \end{eqnarray}
Note that $\{\Delta_{1},...,\Delta_{n}\}$ satisfy  $\sum_{i=1}^n 1/\Delta_{i}^2 = \mathbf{H}$. From (\ref{eqn:pemidstep1}), and by considering only the arm with the smallest gap, 
\begin{eqnarray} \label{eqn:peonly2}
\P(\widehat{i} \neq i) \geq 
 \frac{1}{2}\left(1 -  \prod_{i \neq i^*}  \left(1 -\frac{1}{12} \exp\left({ - m \Delta_{i}^2 }\right)\right) \right) 
 \geq    \frac{1}{24} \exp\left(-\frac{2m}{\mathbf{H}} \right).
\end{eqnarray}
If $m \leq \frac{\mathbf{H}}{2} \log \left(\frac{{1}}{24\delta} \right)$, we have $ \P\left( \widehat{i} \neq i^* \right) \geq \delta$.
This implies that  if the total number of measurements is less than $ \frac{\mathbf{H} n }{2} \log \left(\frac{{1}}{24\delta} \right) $, then $\mathbb{P}(\widehat{i} \neq i^*) \geq \delta$, completing the proof of the second statement of Thm. \ref{thm:NALB1}.   

\end{proof}

\subsection{Proof of Corollary 2}
\begin{proof}
When $\alpha = 0$, Theorem \ref{thm:NALB1} implies the result.  When $\alpha > 0$, we can bound (\ref{eqn:pemidstep1}) by dropping all terms in the product except the term corresponding to the smallest gap.  This gives 
\begin{eqnarray}
\P(\widehat{i} \neq i) \geq 
 \frac{1}{24} \exp\left({ - m \Delta_{1}^2 }\right)  =  \frac{1}{24} \exp\left({ - m n^{-2\alpha} }\right) 
\end{eqnarray}
Setting $m \leq n^{2\alpha} \log\left( \frac{1}{24\delta}\right)$ implies the result.
\end{proof}

\end{document}